\documentclass[twocolumn]{article}

\usepackage[width=17cm,height=22cm]{geometry}
\usepackage[utf8]{inputenc}
\usepackage{fancyvrb}
\usepackage{authblk}
\usepackage{hyperref}
\usepackage{geometry}
\usepackage{graphicx}
\usepackage{tikz}
\usepackage{lipsum,adjustbox}
\usepackage{amsmath,amssymb,amsthm, xpatch}
\usepackage{enumitem}
\usepackage[caption=false]{subfig}
\usepackage[numbers,sort&compress]{natbib}
\usepackage{enumitem}
\usepackage{mathtools} 
\providecommand{\F}{\ensuremath{\Psi}}
\bibpunct[, ]{[}{]}{,}{n}{,}{,}

\makeatletter
\def\NAT@def@citea{\def\@citea{\NAT@separator}}
\makeatother

\newtheorem{theorem}{Theorem}[section]

\newtheorem{corollary}[theorem]{Corollary}
\newtheorem{proposition}[theorem]{Proposition}

\newtheorem{definition}[theorem]{Definition}

\setlist[itemize]{leftmargin=*}
\setlist[enumerate]{leftmargin=*}
\title{Distance and Equivalence between Finite State Machines and Recurrent Neural Networks: Computational results}
\author[1]{Reda Marzouk}
\author[2]{Colin de la Higuera}
\affil[1,2]{Université de Nantes / Laboratoire des Sciences du Numérique de Nantes (LS2N UMR CNRS 6004), Nantes, France}

\begin{document}
\maketitle

\begin{abstract}
  The need of interpreting Deep Learning (DL) models has led, during the past years, to a proliferation of works concerned by this issue. Among strategies which aim at shedding some light on how information is represented internally in DL models, one  consists in extracting symbolic rule-based machines from connectionist models that are supposed to approximate well their behaviour. In order to better understand how reasonable these approximation strategies are, we need to know the computational complexity of measuring the quality of approximation. In this article, we will prove some computational results related to the problem of extracting Finite State Machine (FSM) based models based on trained RNN Language models. More precisely, we'll show the following: (a) For general weighted RNN-LMs with a single hidden layer and a ReLu activation: - The equivalence problem of a PDFA/PFA/WFA and a weighted first-order RNN-LM is undecidable; - As a corollary, the distance problem between languages generated by PDFA/PFA/WFA and that of a weighted RNN-LM is not recursive; -The intersection between a DFA and the cut language of a weighted RNN-LM is undecidable; - The equivalence of a PDFA/PFA/WFA and weighted RNN-LM in a finite support is EXP-Hard; (b) For consistent weight RNN-LMs with any computable activation function: - The Tcheybechev distance approximation is decidable; - The Tcheybechev distance approximation in a finite support is NP-Hard. Moreover, our reduction technique from 3-SAT makes this latter fact easily generalizable to other RNN architectures (e.g. LSTMs/RNNs), and RNNs with finite precision.
\end{abstract}

\medskip

\noindent\textbf{Mots-clef}: Recurrent Neural Networks, Finite State Machines, Distances, Equivalence.

\section{Introduction}
\label{sec:lentete}

    Recurrent Neural Networks and their different variants represent an important family of Deep Learning models suitable to learning tasks with sequential data. However, just like all Deep Learning models in general, this class of models lacks interpretability, which restricts its applicability to highly critical tasks related for instance to security and health, where a formal specification of systems is a mandatory requirement to be approved for deployment in real-case situations.  
   Due to the crucial importance of this limitation, the awareness of providing both expressive and interpretable models keeps growing within the Deep Learning community, resulting in a proliferation of research works focusing on this topic. In general, two major paradigms have been explored in the literature to tackle this issue: 
   \begin{itemize}
       \item \textbf{Interpretable models by design:} In this family of models, the idea is to construct deep learning architectures with the concern of interpretability raised early on the design phase. This change of the architecture may take the form of adding special components to traditional models whose role is to leverage the interpretability issue\cite{Gers00} \cite{Wang19}, using \textit{interpretation-friendly} activation functions, modifying the loss function by injecting a term that favours a resulting interpretable model\cite{Oliva19}, or using a variant of the back-propagation algorithm for training, which enforces that the  resulting model is readily interpretable \cite{Zhang18}. Nevertheless, enforcing the constraint of interpretability by design leads inevitably to a loss of flexibility of the constructed models, and hence a drop of their expressive power. 
       \item \textbf{Post-hoc methods:} Techniques belonging to this family subsume the existence of an already trained DL model and the objective is then to design algorithmic and visualization tools \cite{Karpathy16}\cite{Zeiler14}  that attempt to answer questions related to the interpretability of the original model such as: \\ 
       (a) Semantics of hidden units of the network model: Or, alternatively which role is played by each hidden unit in the network with respect to the learning task (e.g. does a neuron serve as a counter in RNNs trained to learn languages recognized by counter machines\cite{Weiss18a}, a neuron that stores the state of an RNN trained to accept a regular language etc.)? \\ 
       (b) Tracing the causal relationship between the predicted output with respect to the input, also called \emph{instance-level interpretability} \cite{Fong17}\cite{Murdoch18}\cite{Murdoch17}\cite{Du19}; methods form this class raise the problem of local interpretability and, roughly speaking, aims at designing algorithmic answers to the following question: What is the influence degree of each input factor that explains the obtained output?\\ 
       Another important sub-category of post-hoc techniques concerns methods that attempt to extract interpretable rule-based machines (e.g. Decision trees, automata\dots) from DL models \cite{Wang18a}\cite{Wang18b}. Unlike instance-level interpretability techniques, these methods are global and the challenge is how to convert the continuous representation of information as encoded in RNNs into a discrete, symbolic representation, while maintaining a good quality of prediction of these last structures. 
   \end{itemize}
   
   In this work, we are interested in this last family of interpretability methods. More precisely, we address, from a computational viewpoint, the issue of extracting FSM-based machines from general RNN language models.
   
   But this problem would benefit from understanding better how well a finite-state model can approximate an RNN. Which in turn requires solving essential computational problems: can we compute distances between these language models? Can we decide equivalence? These questions have received answers for PFA \cite{lyng02,cort07,higu14b}. We aim in this work to extend these results by including RNN language models into the picture. \\
   
   Our main results are summarized as follows:  (a) For general weighted first-order RNN-LMs with ReLu activation function: 1. The equivalence problem of a PDFA/PFA/WFA and a weighted first-order RNN-LM is undecidable; 2- As a corollary, any distance metric between languages generated by PDFA/PFA/WFA and that of a weighted RNN-LM is also undecidable; -The intersection between a DFA and the cut language of a weighted RNN-LM is undecidable; - The equivalence of a PDFA/PFA/WFA and weighted RNN-LM in a finite support is EXP-Hard; (b) For consistent first-order RNN-LMs with any computable activation function: - The Tcheybetchev distance approximation is decidable; - The Tcheybetchev distance approximation in a finite support is NP-Hard. \\
   
   The rest of this article is organized as follows. Section 2 gives a concise literature overview of issues related to our problematic. Section 3 presents our results for the case of general first-order RNN language models(RNN-LM) with ReLu activation function. Section 4 is dedicated to the case of consistent RNN-LMs. 
   \section{Related works}
     The problem of symbolic knowledge extraction from connectionist models is not a new issue, and one can trace back works interested in this problem since the development of the first neural architectures \cite{Minsky67}\cite{McCulloch43}. However, with the development of novel spatio-temporal connectionist models in the nineties, the most important of which is Ellman RNNs\cite{Elman90}, and their great empirical success on inferring language models with limited amount of data and with performance results that often outscore rule-based algorithms traditionally used in the Grammatical Inference field \cite{higu10}, research interests in this issue has regained more attention. In fact, these works were mostly driven by a legitimate motivation: if an RNN-like structure is trained to recognize a language belonging to a given class of languages $\mathcal{C}$, and this latter can be recognized by a class of \textit{computing devices} $\mathcal{M}$, then there must be a close connection between the representation of the target language as encoded in the RNN-like structure on one hand, and that of the corresponding computing device in $\mathcal{M}$ that is capable of recognizing it on the other. \\ 
     This aforementioned motivation raises two fundamental questions, at least from a theoretical viewpoint: 
     \begin{enumerate}
         \item What is the expressive power of different classes of \textit{``RNN Machines''}, as compared to classical symbolic machines (e.g. deterministic/non deterministic finite state automata, deterministic/non determinstic pushdown automata, Turing machines etc.)?
         \item How can we design algorithms that extract symbolic machines from RNN models? What are the theoretical guarantees of such methods? What is the computational complexity of such problems? 
     \end{enumerate}
     We should note that these two questions are, in some sense, interrelated. If a class of RNNs is very powerful --say Turing Equivalent-- computational problems related to the extraction of finite state machines are more likely to be undecidable. In fact, as a corollary of Rice's Theorem\footnote{The Rice Theorem states that any class of non-trivial languages recognized by a Turing machine is not recursive}, the equivalence between a Turing machine and any non-trivial class of computing devices is necessarily undecidable, which means in practice that no algorithm can exist that can answer the question of equivalence between symbolic machines and ``Turing- Equivalent'' RNN ones. In other words, from the perspective of the theory of computation, the trade-off between expressiveness and interpretability\footnote{In our context, we quantify the interpretability of a model as a measure of the computational difficulty by which one can extract a finite state machine. A more rigorous formal definition of what is an \textit{intepretable model} is still an arguably open question.} in connectionist models is unavoidable. As a consequence of the above discussion, we argue that analyzing a class of RNNs as a computational model can give many insights with regard to its interpretability.  \\
     Guided by questions raised above, we divide the rest of this section into two parts: In the first part, we examine works present in the literature that focused on the computational power of recurrent neural networks and its consequences on some computational problems concerning RNNs. In the second part, we give a brief overview of existing methods in the literature aimed at extracting finite state machines from trained RNN ones. 
     \subsection{Computational power of RNNs}
     The question of the computational capabilities of different classes of RNN has been addressed since the early development of neural systems. To the best of the author's knowledge, early works that addressed this dates back to the middle of the previous century by McCulloch \textit{et al.} \cite{McCulloch43} and Kleene\cite{Kleene56}, where it was proven that networks with binary threshold activation functions are capable of implementing finite state automata. In \cite{Pollack87}, Pollack designed a \emph{Turing-complete} class of high-order recurrent neural networks with two types of activation function (linear and Heaviside). This result was later extended in \cite{Siegelmann95}, where authors relaxed the high-order requirement, and showed that first-order RNNs with saturated-linear activation functions were \textit{Turing Complete}. Later on, Kilian et el. generalized this result to sigmoidal activation functions \cite{Kilian93}. \\
      The Turing Completeness of some classes of RNNs has many consequences with respect to the computational class to which belong many problems related to them. In \cite{Chen18}, the authors proved that the problem of deciding whether a RNN language model -RNN-LM- with ReLu activation function is consistent (encodes a valid probability distribution) or not is an undecidable problem. Moreover, the consensus string problem and finding a minimal RNN-LM equivalent to a given RNN-LM or testing the equivalence between two RNN-LMs are also undecidable. \\ 
      
       Given these pessimistic results about computability of several important problems related to RNNs, a new line of research suggests to analyze the practical capabilities computational power of neural nets instead of the classical \textit{``unrealistic''} theoretical model, by constraining the amount of memory resources of the RNN hidden units to be finite \cite{Weiss18a}\cite{Merrill19}. Under this constraint, Korskky \textit{et al.} \cite{Korsky19} proved that RNNs with one hidden layer and ReLu activation, and GRUs are expressively equivalent to deterministic finite automata. In \cite{Weiss18a}, Weiss \textit{et al.} showed that the class of finite precision LSTMs were able to simulate counter machines, while the simple class of Elman RNNs and GRUs can't.
      
      \subsection{Extraction of automata-based machines from trained RNNs}
      Early works investigating the problem of extracting automata-based machines from trained RNNs coincide with the emergence of novel RNN architectures \cite{Elman90}\cite{Giles91} in early nineteens that have shown promising results for the task of inferring language models from limited data. These early works have mainly focused on the extraction of deterministic finite automata (DFAs) from RNNs trained to recognize regular languages, and most of which were based on the assumption that a well-trained RNN to recognize a regular language tend to group hidden states of the RNN into clusters that maps directly to states of the minimal DFA recognizing the target regular language. Based on this assumption, the problem of DFA extraction from RNNs boils down to a clustering/quantization problem of the RNN's hidden state space, and many clustering techniques were proposed for this task: Quantization by Equipartition \cite{Giles92}\cite{Watrous92}, Hierarchical Clustering \cite{Alquezar94}, k-means \cite{Zeng93}\cite{Schellhammer89}, fuzzy clustering \cite{Cechin03} etc. \\
      During the last few years, as RNN-based architectures became more sophisticated and thus harder to be a subject of interpretative analysis, the issue has gained an increasing interest among researchers, and new methods were proposed in the literature to extract automata-based machines from different classes of RNNs. In \cite{Weiss18b}, Weiss \textit{et al.} proposed an adaptation of the $L*$ algorithm \cite{Angluin87} to extract deterministic finite automata (DFA) from RNN Acceptors, where an RNN Acceptor model serves as a black box oracle for \textit{approximate} equivalence and membership queries, hinting that the exact equivalence query is \textit{``likely to be intractable''}.  Same authors extended their work in \cite{Weiss18b} to extract Probabilistic Deterministic Finite Automata from RNN-LMs. In order to answer the equivalence query, authors used a sampling strategy of both models, and gave theoretical guarantees of its convergence in probability under a relaxed notion of equivalence. Ayache \textit{et al.} \cite{Ayache19} employed the spectral learning framework \cite{Balle14} to extract Weighted Finite Automata(WFA) from a RNN language model. In \cite{Okudono19}, Okudono \textit{et al.} raised the problem of answering the equivalence query between a RNN language model and a WFA proposing an empirical regression-based technique to perform this task. However, no theoretical guarantees were provided to back their method.

\section{Definitions and Notations}
Let $\Sigma$ be a finite alphabet. The set of all finite strings is denoted by $\Sigma^{*}$. The set of all strings whose size is equal (resp. greater than or equal) to $n$ is denoted by $\Sigma^{n}$(resp. $\Sigma^{\geq n}$). For any string $w \in \Sigma^{*}$, the size of $w$ is denoted by $|w|$, and its $n$-th symbol by $w_{n}$. The prefix of length $n$ for any string $w \in \Sigma^{\geq n}$ will be referred to as $w_{:n}$. The symbol $\$$ denotes a special marker. The symbol $\Sigma_{\$}$ will refer to the set $\Sigma \bigcup \{\$\}$. \\ 

\textbf{Weighted languages:} A weighted language f over $\Sigma$ is a mapping that assigns to each word $w \in \Sigma^{*}$ a weight $f(w) \in \mathbb{R}$. A WL $f$ is called consistent, if it encodes a valid probability distribution, i.e. satisfies the following properties: $\forall w \in \Sigma^{*}: f(w) \geq 0,~\sum\limits_{w \in \Sigma^{*}}f(w) = 1$. Two WLs $f_{1},~f_{2}$ are said to be equivalent if: $\forall w \in \Sigma^{*}:~f_{1}(w)=f_{2}(w)$. The Tcheybetchev distance metric between two WLs is denoted $d_{\infty}(f_{1},f_{2})$, and defined as: $\max\limits_{w \in \Sigma^{*}} |f_{1}(w) - f_{2}(w)|$. Finally, we define, for a given scalar $c > 0$, the cut-point language of $f$ with respect to $c$ and denoted $\mathcal{L}_{f,c}$, as the set of finite words whose values are greater or equal to $c$. \\

In Section 4, a 3-SAT formula will be denoted by the symbol $\F$. A formula is comprised of $n$ Boolean variables denoted $x_{1},..x_{n}$, and $k$ clauses $C_{1},..C_{k}$. For each clause, we'll use notation $l_{i1},~l_{i2},~l_{i3}$ to refer to its three composing literals. For a given string $w \in \{0,1\}^{n}$, the number of clauses satisfied by $w$ will be denoted by $N_{w}$. \\ 

For the rest of this section, we shall first provide a formal definition of the class of first-order weighted RNN-LMs that we'll study in this work. Also, we'll give a brief recall of basic definitions of different automata-based machines that we'll encounter throughout the rest of this article. 
\begin{definition}\cite{Chen18}
 A First-order weighted RNN Language model is a weighted language $f: \Sigma^{*} \rightarrow \mathbb{R}$ and is defined by the tuple $<\Sigma, N, h^{(0)}, \sigma, W, (W')_{\Sigma_{\$}}, E, E'>$ such that:
 \begin{itemize}[noitemsep]
     \item $\Sigma$ is the input alphabet,
     \item $N$ the number of hidden neurons,
     \item $\sigma: \mathbb{Q} \rightarrow \mathbb{Q}$ is a computable activation function, 
     \item $W \in \mathbb{Q}^{N \times N}$ is the state transition matrix,
     \item $\{W'_{\sigma}\}_{\sigma \in \Sigma_{\$}}$, where each $W'_{\sigma} \in \mathbb{Q}^{N}$ is the embedding vector of the symbol $\sigma \in \Sigma_{\$}$,
     \item $O \in \mathbb{Q}^{\Sigma_{\$} \times N}$ is the output matrix,
     \item $O' \in \mathbb{Q}^{\Sigma_{\$}}$ the output bias vector.
 \end{itemize}
 The computation of the weight of a given string $w$ (where $\$$ is the end marker) by $R$ is given as follows.
 (a) Recurrence equations:
 $$ h^{(t+1)} = \sigma(W.h^{(t)} + W'_{w_{t}})$$
 $$ E_{t+1} = O h^{(t+1)} + O'$$
 $$ E'_{t+1} = softmax_{2}(E_{t+1})$$
 (b) The resulting weight:
 $$ R(w) = \prod\limits_{i=0}^{|w|+1} E'_{i}$$
 where $w_{0} = w_{|w|+1} = \$$
\end{definition}
Remark that, in order to avoid technical issues, we used softmax base 2 defined as: $softmax_{2}(x)_{i} = \frac{2^{x_{i}}}{\sum\limits_{j=1}^{n} 2^{x_{j}}}$ for any $x \in \mathbb{R}^{d}$ instead of the standard softmax in the previous definition.  
In the following, hidden units of the network will be designated by lowercase letters $n_{1},n_{2},..$, and their activations at time $t$ by $h_{n}^{t}$. Also, we denote by $\mathcal{R}_{\sigma}$ the class of RNN-LMs when $\sigma$ is the activation function. For example, an important class of RNN-LMs that will be used extensively in the rest of the article is $\mathcal{R}_{ReLu}$. \\

\textbf{Weighted Finite Automata (WFA).} WFAs represent weighted versions of nondeterministic finite automata, where transitions between states, denoted $\delta(q,\sigma,q')$ where $q,q' \in Q$ represents states of the WFA are labeled with a rational weight $T(q,\sigma,q')$, and each of its nodes $q \in Q$ is labeled by a pair of rational numbers $(I(q),P(q))$ that represents respectively the initial-state and final-state weight of $q$. WFAs model  weighted languages where the weight of a string $w$ is equal to the sum of the weights of all paths whose transitions encode the string $w$. The weight of a path $p$ is calculated as the product of the weight labels of all its transitions, multiplied by the initial-state weight of its staring node and the final-state weight of its ending node. \\ 

\textbf{Probabilistic Finite Automata (PFA).} A PFA is a WFA with two additional constraints: First, the sum of initial-state weights of all states is a valid probability distribution over the state space. Second, for each state, the sum of weights of its outcoming edges added to its finite-state weight is equal to $1$. This additional constraint restricts the power of PFAs to encode stochastic languages \cite{Thollard05}, which makes it useful for representing language models. Interestingly, PFAs are proven to be equivalent to Hidden Markov Models (HMMs), and the construction of equivalent HMMs from PFAs and vice versa can be done in polynomial time\cite{Vidal05}. The deterministic version of PFAs, a.k.a Deterministic Probabilistic Finite Automata (DPFA), enforces the additional constraint that for any state $q$, and for any symbol $\sigma$ there is at most one outgoing transition labeled by $\sigma$ from $q$. \\

\section{Computational results for general RNN-LMs with ReLu activation functions}
 The choice of $ReLu$ in this part of the article is not arbitrary. In fact, due to its nice piecewise-linear property and its wide use in practice, the $ReLu(.)$ function is first choice to analyze theoretical properties of RNN architectures. Recently, Chen et al. \cite{Chen18} provided an extensive study of first-order RNN Language Models with $ReLu$ as an activation function from a computational viewpoint. Weiss \textit{et al.} \cite{Weiss18a} proved both theoretically and empirically that first-order RNNs with $ReLu(.)$ can simulate counter machines. Korsky \textit{et al.}\cite{Korsky19} proved that finite precision first-order RNNs with $ReLu$ are computationally equivalent to deterministic finite automata. Moreover, when allowed arbitrary precision, they can simulate pushdown automata. Analyzing RNNs with other widely used activation functions, such as the sigmoid and the hyperbolic tangent, are left for future research. 
\subsection{Turing Completeness of general weighted-RNNs: Siegelmann's construction}
The basic building block for proving computational results presented in this part of the article is the work done by Siegelmann and al. in \cite{Siegelmann95} to prove the Turing completeness of a certain class of first-order RNNs. Hence, we propose, in this section, to provide a global scope of this construction, followed by an equivalent reformulation of their main theorem that will be relevant for our work. \\ 
 The main intuition of Siegelmann \textit{et al.}'s work is that, with an appropriate encoding of binary strings, a first-order RNN with a saturated linear function can readily simulate a stack datastructure by making use of a single hidden unit. For this, they used 4-base encoding scheme that represents a binary string $w$ as a rational number: $Enc(w) = \sum\limits_{i=1}^{|w|} \frac{w_{i}}{4^{i}}$. Backed by this result, they proved than any two-stack machine can be simulated by a first-order RNN with linear saturated function, where the configuration of a running two-stack machine (i.e. the content of the stacks and the state of the control unit) is stored in the hidden units of the constructed RNN.  Finally, given that any Turing Machine can be converted into an equivalent two-stack machine (The set of two-stack machines is Turing-complete \cite{Hopcroft06}), they concluded their result. \\ 
 In the context of our work, two additional remarks need to be noted about Siegelmann's construction: - First, although the class of first-order RNNs examined in their work uses the saturated linear function as an activation function, their result is generalizable to the ReLu activation function (or, more generally, any computable function that is linear in the support [0,1])? - Second, although not mentioned in their work, the construction of the RNN from a Turing Machine is polynomial in time. In fact, on one hand, the number of hidden units of the constructed RNN is linear in the size of the Turing Machine, and the construction of transition matrices of the network is also linear in time. On the other hand, notice that the 4-base encoding map $Enc(.)$ is also computable in linear time. \\ 
 In light of these remarks, we are now ready to present the following theorem:
 \begin{theorem}({Theorem 2, \cite{Siegelmann95})}
   Let $\phi: \{0,1\}^{*} \rightarrow \{0,1\}^{*}$ be any computable function, and $M$ be a Turing Machine that implements it. We have, for any binary string $w$, there exists $N = O(poly(|M|)),~h^{(0)} = [Enc(w)~~0..0] \in \mathbb{Q}^{N},~W \in \mathbb{Q}^{N \times N}$, such that for any finite alphabet $\Sigma$, $~\forall \sigma \in \Sigma_{\$}:~W'_{\sigma} \in \mathbb{Q}^{N}, O \in \mathbb{Q}^{|\Sigma_{\$}| \times N},~ O' \in \mathbb{Q}^{|\Sigma_{\$}|}$, $R = <\Sigma,N,ReLu,W,W',O,O'> \in \mathcal{R}_{ReLu}$
    verifies:
   \begin{itemize}
       \item if $\phi(w)$ is defined, then there exists  $T \in \mathbb{N}$ such that the first element of the hidden vector $h_{T}$ is equal to $Enc(\phi(w))$, and the second element is equal to $1$,
       \item if $\phi(w)$ is undefined (i.e. $M$ never halts on $w$), then for all $t \in \mathbb{N}$, the second element of the hidden vector $h_{t}$ is always equal to zero.  
   \end{itemize}
   Moreover, the construction of $h_{0}$ and $W$ is polynomial in $|M|$ and $|w|$. 
 \end{theorem}
 In the following, we'll denote by $\mathcal{R}_{ReLu}^{M,w}$ the set of RNNs in $\mathcal{R}_{ReLu}$ that simulate the TM $M$ on $w$. It is important to note that the construction of a RNN that simulates a TM on a given string in the previous theorem is both input and output independent. The only constraints that are enforced by the construction are placed on the transition dynamics of the network and the initial state. In fact, the input string is \textit{placed} in the first stack of the two-stack machine before running the computation (i.e. in the initial state $h^{(0)}$). Under this construction, the first stack of the machine is encoded in the first hidden unit of the network. Afterwards, the \textit{RNN Machine} runs on this input, and halts(If It ever halts) when the halting state of the machine is reached. In theorem 1.1, the halting state of the machine is represented by the second neuron of the network. In the rest of the article, we'll refer to the neuron associated to the halting state by the name \textit{halting neuron}, denoted $n_{halt}$. \\
  We present the following corollary that gives a characterization of the halting machine problem\footnote{The Halting Machine problem is defined as follows: Given a TM M and a string w, does M halt on w? This problem is undecidable.} that relates it to the class $\mathcal{R}_{ReLu}$:
 \begin{corollary}
   Let $M$ be any Turing Machine, and $w$ be a binary string, $M$ halts on $w$ if and only if for any $R \in \mathcal{R}_{Relu}^{M,w}$ , there exists $T \in \mathbb{N}$, such that $\forall t < T: h_{n_{halt}}^{(t)}=0$, and $h_{n_{halt}}^{(T)}=1$.
 \end{corollary}
 
 \subsection{The equivalence problem between FSAs and general RNNs}
 The equivalence problem between a DPFA and a general RNN-LMs is formulated as follows: \\
 \noindent \textbf{Problem.} Equivalence Problem between a DPFA and a general RNN \\
 \textit{Given a general weighted RNN-LM $R \in \mathcal{R}_{ReLu}$ and a DPFA $\mathcal{A}$. Are they equivalent?}
 \begin{theorem}
 The equivalence problem between a DPFA and a general RNN is undecidable
 \end{theorem}
 \begin{proof}
 We'll reduce the halting Turing Machine problem to the Equivalence problem.
   Let $\Sigma=\{a\}$. We first define the trivial DPFA $\mathcal{A}$ with one single state $q_{0}$, and $T(\delta_{q_{0},a,q_{0}})=P(q_{0})=\frac{1}{2},~I(q_{0})=1$. This DPFA implements the weighted language $f(a^{n})=\frac{1}{2^{n+1}}$. \\
   Let $M$ be a Turing Machine and $w \in \Sigma^{*}$. Let $R \in \mathcal{R}_{ReLu}^{M,w}$ such that $O[n_{halt},a]=1$ ,0 everywhere and $O'$ is equal to zero everywhere. We construct a RNN $R'$ from $R$ by adding one neuron in the hidden layer, denoted $n'$ such that: $h_{n'}^{(0)} = 0,~\forall t \geq 0:~h_{n'}^{(t+1)}= ReLu(h_{n'}^{(t)}),~ O[n',\$]=1$. \\
   Notice that, by Corollary 4.2, the TM $M$ never halts on w if and only if $\forall T: (h_{n_{halt}}^{(T)},h_{n'}^{(T)})=(0, 0)$, i.e. $R(a^{n}\$) = \frac{1}{2^{n+1}}$. That is, the TM $M$ doesn't halt on $w$ if and only if the DPFA $\mathcal{A}$ is equivalent to $R'$, which completes the proof.
 \end{proof}
 
 A direct consequence of the above theorem is that the equivalence problem between PFAs/WFAs and general RNN-LMs in $\mathcal{R}_{ReLu}$ is also undecidable, since the DPFA problem case is immediately reduced to the general case of PFAs (or WFAs). Another important consequence is that no distance metric can be computed between  DPFA/PFA/WFA and $\mathcal{R}_{ReLu}$:
 
 \begin{corollary}
   Let $\Sigma = \{a\}$. For any distance metric $d$ of $\Sigma^{*}$, the total function that takes as input a description of a PDFA $\mathcal{A}$ and a general RNN-LM $\mathcal{R}_{ReLu}$ and outputs $d(\mathcal{A}, R) $ is not recursive. \\
   This fact is also true for PFAs and WFAs.
 \end{corollary}
\begin{proof}
  Let $d$ be any distance metric on $\Sigma^{*}$. By definition of a distance, we'll have $d(\mathcal{A}, R) = 0$ if and only if $\mathcal{A}$ and $R$ are equivalent. Since the equivalence problem is undecidable, $d(.)$ can't be computed.
\end{proof}
 
 \subsection{Intersection of the cut language of a general weighted RNN-LM with a DFA}
 In this subsection, we are interested in the following problem: \\
\noindent\textbf{Problem.} Intersection of a DFA and the cut-point language of a weighted RNN-LM  \\
 \textit{ Given a general weighted RNN-LM $R \in \mathcal{R}_{ReLu}$, $c \in \mathbb{Q}$, and a DFA $\mathcal{A}$, is $\mathcal{L}_{R,c} \bigcup \mathcal{L}_{\mathcal{A}} = \emptyset$?}
 Before proving that this problem is undecidable, we shall recall first a result proved in \cite{Chen18}:
 \begin{theorem}{(Theorem 9, \cite{Chen18})}
 Define the highest-weighted string problem as follows: Given a weighted RNN-LM $R \in \mathcal{R}_{ReLu}$, and $c \in (0,1)$: Does there exist a string $w$ such that $R(w)>c$? \\
 The highest-weighted  string problem is undecidable. This problem is also known as the consensus problem \cite{lyng02} and it is known to be NP-hard even for PFA.
 \end{theorem} 
 \begin{corollary}
 The intersection problem is undecidable.
 \end{corollary} 
 \begin{proof}
 We shall reduce the highest-weighted string problem from the intersection problem. Let $R \in \mathcal{R}_{ReLu}$ a general weighted RNN-LM, and $c \in (0,1)$. Construct the automaton $\mathcal{A}$ that recognizes $\Sigma^{*}$. We have that $\mathcal{L}_{\mathcal{A}} \bigcap \mathcal{L}_{R} = \mathcal{L}_{R} = \emptyset$ if and only if there exist no string $w$ such that $R(w)>c$, which completes the proof.
 \end{proof}
 
 \subsection{The equivalence problem in finite support}
 Given that the equivalence problem between a weighted RNN-LM and different classes of finite state automata is undecidable, a less ambitious goal is to decide whether a RNN-LM agrees with a finite state automaton over a finite support. We formalize this problem as follows: \\
  \textbf{Problem.} The EQ-Finite problem between PDFA and weighted RNN-LMs \\
 \textit{ Given a general weighted RNN-LM $R \in \mathcal{R}_{ReLu}$, $m \in \mathbb{N}$ and a PDFA $\mathcal{A}$. Is $R$ equivalent to $\mathcal{A}$ over $\Sigma^{\leq m}$? }
 
 \begin{theorem}
  The EQ-Finite problem is EXP-Hard.
 \end{theorem}
 \begin{proof}
 We reduce the bounded halting problem \footnote{The bounded halting problem is defined as follows: Given a TM M, a string $x$ and an integer $m$, encoded in binary form. Decide if M halts on $x$ in at most $n$ steps? This problem is EXP-Complete.} to the EQ-Finite problem. \\
 The proof is similar to the used for Theorem 4.3.  We are given a TM $M$, a string $w$ and $m \in \mathbb{N}$. Let $\Sigma = \{a\}$. We construct a general weighted RNN-LM $R'$ by augmenting $R \in \mathcal{R}_{ReLu}^{M,w}$ with a neuron $n'$ as in Theorem 4.3. By Theorem 4.1, this reduction runs in polynomial time. On the other hand, let $\mathcal{A}$ be the trivial PDFA with one single state $q_{0}$, and $T(\delta_{q_{0},a,q_{0}})=P(q_{0})=\frac{1}{2},~I(q_{0})=1$. Note that $R'$ doesn't halt in $m$ steps if and only if $\forall T \leq m:~(n_{halt}^{(T)},n'^{(T)})=(0, 0)$, i.e. $R'(a^{n}\$) = \frac{1}{2^{n+1}}$ for the first $m$ running steps on $R'$, in which case the language modeled by $R'$ is equal to $f$ in $\Sigma^{\leq m}$. Hence, $\mathcal{A}$ is equivalent to $R$ in $\Sigma^{\leq m}$ if and only if $M$ doesn't halt on the string $w$ in less or equal than $m$ steps.
 \end{proof}

\section{Computational results for consistent RNN-LMs with general activation functions}
In the previous section, we have seen that many interesting questions related to measuring the similarity between weighted languages represented by different classes of weighted automata and first-order RNN-LMs with ReLu activation function turned out to be either undecidable, or intractable when restricted to finite support. In this section, we examine the case where trained RNN-LMs are guaranteed to be consistent, and we raise the question of approximate equivalence between PFAs and first-order consistent RNN-LMs with \textit{general} computable activation functions.  For any computable activation function $\sigma$, we formalize this question in the following two decision problems:  \\

\noindent\textbf{Problem.} \textit{Approximating the Tchebychev distance between RNN-LM and PFA} \\
\textbf{Instance:} A consistent RNN-LM $R \in \mathcal{R}_{\sigma}$, a consistent PFA $\mathcal{A}$, $c>0$ \\
 \textbf{Question:} Does there exist $|w| \in \Sigma^{*}$ such that $|R(w) - \mathcal{A}(w)| > c$ \\

\noindent\textbf{Problem.} \textit{Approximating the Tchebychev distance between consistent RNN-LM and PFA over finite support} \\
\textbf{Instance:} A consistent RNN $R \in \mathcal{R}_{\sigma}$, a consistent PFA $\mathcal{A}$, $c>0$ and $N \in \mathbb{N}_{+}$, \\
 \textbf{Question:} Does there exist $|w| \leq N$ such that $|R(w) - \mathcal{A}(w)| > c$ \\
 Note that there is no constraint on the activation function used for consistent RNN-LMs in these defined problems, provided it is computable. The first fact is easy to prove:
 \begin{theorem}
 Approximating the Tcheybechev distance between RNN-LM and PFA is decidable.
 \end{theorem}
 \begin{proof}
  Let $R$ be a consistent RNN-LM and $\mathcal{A}$ be a consistent PFA. An algorithm that can decide this problem runs as follows: enumerate all strings $w_{1},..$ in $\Sigma^{*}$ until we reach a string that satisfies this property in which case the algorithm returns Yes. If there is no such string, by definition of consistency, there will be a finite time $T$ such that $\sum\limits_{t=1}^{T} R(w_{t}) \geq 1-c,~ \sum\limits_{t=1}^{T} \mathcal{A}(w_{t}) \geq 1-c$ in which case, we have: $\forall t >T:~ R(w_{t}) < c$ and $\mathcal{A}(w_{t}) < c$ which implies $\forall t > T: |R(w_{t} - \mathcal{A}(w_{t})| < c$. When $T$ is reached, the algorithm returns No.
 \end{proof}
 \subsection{Approximating the Tcheybetchev distance over a finite support}
 Proving the NP-Hardness of the Tcheybetchev distance approximation in finite support is more complicated, and we'll give below the construction of a PFA and a RNN from a given 3-SAT formula which will help us prove the result. Let $\epsilon \in (0,\frac{1}{2})$ whose value will be specified later. \\ 
 
$\bullet$ \textbf{Construction of a PFA $\mathcal{A}$: } The construction of our PFA is inspired from the work done in \cite{Casacuberta00}, and illustrated in Figure 1. Intuitively, each clause $i$ in $\F$ is represented by two paths in the PFA, one that encodes a satisfiable assignment of the variables for this clause, and the other not. More formally, the PFA $\mathcal{A}$ is defined as:
 \begin{itemize}
     \item $Q_{\mathcal{A}} = \{q_{0}\} \cup \{q_{i,j}^{c}:~ i \in [1,k],~j \in [1,n],~ c \in \{T,F\}  \}$ is the set of states, 
     \item Initial probabilities: $I_{\mathcal{A}}(q_{0})=1,~0$ otherwise,
     \item For every $i \in [1,k],~c \in \{T,F\}$ and $a \in \Sigma$: $(q_{0},a,q_{i,1}^{c}) \in \delta_{\mathcal{A}}$, 
     \item $\forall i \in [1,k],~j \in [2,n-1], a \in \Sigma: ~(q_{i,j}^{T},a,q_{i,j+1}^{T}) \in \delta_{\mathcal{A}}$
     \item For each clause $i$: 
     \begin{itemize}
         \item If $x_{1} \in \{l_{i1},l_{i2},l_{i3}\}$, then $\forall a \in \Sigma:~ (q_{i,1}^{T},a,q_{i,2}^{T}) \in \delta_{\mathcal{A}}$. And: 
         \begin{itemize}
             \item If $x_{2} \in \{l_{i1},~l_{i2},~l_{i3}\}$, then $(q_{i1}^{F},1,q_{i2}^{T}) \in \delta_{\mathcal{A}}$, and $(q_{i1}^{F},0,q_{i2}^{F}) \in \delta_{\mathcal{A}}$,
             \item If $\bar{x}_{2} \in \{l_{i1},~l_{i2},~l_{i3}\}$, then $(q_{i1}^{F},0,q_{i2}^{T}) \in \delta_{\mathcal{A}}$, and $(q_{i1}^{N},1,q_{i2}^{N}) \in \delta_{\mathcal{A}}$
             \item Otherwise $a \in \Sigma: (q_{i1}^{F},a,q_{i2}^{F}) \in \delta_{\mathcal{A}}$
         \end{itemize}
         \item If $\bar{x}_{1} \in \{l_{i1},~l_{i2},~l_{i3}\}$, then $\forall a \in \Sigma:~(q_{i1}^{F},a,q_{i2}^{T}) \in \delta_{\mathcal{A}}$, and:
         \begin{itemize}
             \item If $x_{2} \in \{l_{i1},~l_{i2},~l_{i3}\}$, then $(q_{i1}^{T},1,q_{i2}^{T}) \in \delta_{\mathcal{A}}$, and $(q_{i1}^{T},0,q_{i2}^{F}) \in \delta_{\mathcal{A}}$,
             \item If $\bar{x}_{2} \in \{l_{i1},~l_{i2},~l_{i3}\}$, then $(q_{i1}^{T},0,q_{i2}^{T}) \in \delta_{\mathcal{A}}$, and $(q_{i1}^{T},1,q_{i2}^{F}) \in \delta_{\mathcal{A}}$ \item Otherwise, $\forall a \in \Sigma:~ (q_{i1}^{T},a,q_{i2}^{F}) \in \delta_{\mathcal{A}}$
         \end{itemize}
         \item Otherwise $\forall a \in \Sigma,~c \in \{T,F\}: (q_{i,1}^{c},a,q_{i,2}^{c}) \in \delta_{\mathcal{A}}$
     \end{itemize}
     \item For each clause $i$ and every Boolean variable $x_{j}$ where $j \in [2,n-1]$:
     \begin{itemize} 
     \item if $x_{j} \in \{l_{i1},~l_{i2},~l_{i3}\}$, then $(q_{i,j}^{N},1, q_{i,j+1}^{S}) \in \delta_{\mathcal{A}}$
     \item if $\bar{x_{j}}  \in \{l_{i1},~l_{i2},~l_{i3}\}$, then $(q_{i,j}^{N},0, q_{i,j+1}^{T}) \in \delta_{\mathcal{A}}$
     \item Otherwise, $\forall a \in \Sigma:~(q_{i,j}^{N},a,q_{i,j+1}^{N}) \in \delta_{\mathcal{A}} $
     \end{itemize}
     \item Transition probabilities: 
     \begin{itemize}
         \item $\forall i \in [1,k],~a \in \Sigma,~ c \in \{S,N\}:~~T_{\mathcal{A}}(q_{0},a,q_{i1}^{c}) = \frac{1}{2k} - \frac{\epsilon}{k}$
         \item $\forall i \in [1,k],~a \in \Sigma:~T_{\mathcal{A}}(q_{in}^{T},a,q_{in}^{F})=\epsilon$
         \item All the other transitions belonging to $\delta_{\mathcal{A}}$ has a weight $\frac{1}{2} - \epsilon$
     \end{itemize}
     \item Final-state probabilities:
     \begin{itemize}
         \item For each clause $i$: $P_{\mathcal{A}}(q_{in}^{N}) = 1 - 2\epsilon$
         \item All the other states in $\mathcal{A}$ has a final-state probability equal to $2 \epsilon$
     \end{itemize}
 \end{itemize}

\begin{figure*}
\begin{center}
\begin{tikzpicture}
\tikzset{vertex/.style = {shape=circle,draw,minimum size=1.5em}}
\tikzset{edge/.style = {->,> = latex'}}
\node[vertex,label=above:$q_{0}$] (q0) at  (0,0) {$1/2 \epsilon$};
\node[vertex,label=above:$q_{11}^{T}$] (q11t) at  (3,2.5) {$0/2 \epsilon$};
\node[vertex,label=above:$q_{12}^{T}$] (q12t) at  (6,2.5) {$0/2 \epsilon$};
\node[vertex,label=above:$q_{13}^{T}$] (q13t) at  (9,2.5) {$0/2 \epsilon$};
\node[vertex,scale = 0.8,label=right:$q_{14}^{T}$] (q14t) at  (12,2.5) {$0/ 1-2 \epsilon$};
\node[vertex,label=above:$q_{11}^{F}$] (q11f) at  (3,0.5) {$0/2 \epsilon$};
\node[vertex,label=above:$q_{12}^{F}$] (q12f) at  (6,0.5) {$0/2 \epsilon$};
\node[vertex,label=above:$q_{13}^{F}$] (q13f) at  (9,0.5) {$0/2 \epsilon$};
\node[vertex,label=right:$q_{14}^{F}$] (q14f) at  (12,0.5) {$0/2 \epsilon$};
\node[vertex,label=above:$q_{21}^{T}$] (q21t) at  (3,-1.5) {$0/2 \epsilon$};
\node[vertex,label=above:$q_{22}^{T}$] (q22t) at  (6,-1.5) {$0/2 \epsilon$};
\node[vertex,label=above:$q_{23}^{T}$] (q23t) at  (9,-1.5) {$0/2 \epsilon$};
\node[vertex, scale = 0.8, label=right:$q_{24}^{T}$] (q24t) at  (12,-1.5) {$0/ 1-2 \epsilon$};
\node[vertex,label=above:$q_{21}^{F}$] (q21f) at (3,-3.5) {$0/2 \epsilon$};
\node[vertex,label=above:$q_{22}^{F}$] (q22f) at  (6,-3.5) {$0/2 \epsilon$};
\node[vertex,label=above:$q_{23}^{F}$] (q23f) at  (9,-3.5) {$0/2 \epsilon$};
\node[vertex,label=right:$q_{24}^{F}$] (q24f) at  (12,-3.5) {$0/ 2 \epsilon$};

\draw[->] (q0) -- (q11t) node[midway, above, rotate=45] {$1: \frac{1}{4} - \frac{\epsilon}{2}$};
\draw[->] (q0) -- (q11f)  node[midway, below,rotate=10] {$0: \frac{1}{'} - \frac{\epsilon}{2}$};
\draw[->] (q0) -- (q21t) node[midway, below, rotate=-25]{$1: \frac{1}{4} - \frac{\epsilon}{2}$} ;
\draw[->] (q0) -- (q21f) node[midway, below, rotate=-45]  {$0: \frac{1}{4} - \frac{\epsilon}{2}$};

\draw[->] (q11t) -- (q12t) node[midway,above] {$0,1: \frac{1}{2} - \epsilon$};
\draw[->] (q12t) -- (q13t) node[midway,above] {$0,1: \frac{1}{2} - \epsilon$};
\draw[->] (q13t) -- (q14t) node[midway,above] {$0,1: \frac{1}{2} - \epsilon$};
\draw[->] (q11f) -- (q12t) node[midway,above, rotate=30] {$1: \frac{1}{2} - \epsilon$};
\draw[->] (q11f) -- (q12f) node[midway,above] {$0: \frac{1}{2} - \epsilon$};
\draw[->] (q12f) -- (q13t) node[midway,above,rotate=30] {$1: \frac{1}{2} - \epsilon$};
\draw[->] (q12f) -- (q13f) node[midway,above] {$0: \frac{1}{2} - \epsilon$};
\draw[->] (q13f) -- (q14t) node[midway,above,rotate=30] {$1: \frac{1}{2} - \epsilon$};
\draw[->] (q13f) -- (q14f) node[midway,above] {$0: \frac{1}{2} - \epsilon$};
\draw[->] (q14t) -- (q14f) node[midway,left] {$0,1: \epsilon$};
\path (q14f) edge [out=330,in=300,looseness=8] node[right] {$0,1: \frac{1}{2} - \epsilon$} (q14f);

\draw[->] (q21t) -- (q22t) node[midway,above] {$0: \frac{1}{2} - \epsilon$};
\draw[->] (q21t) -- (q22f) node[midway,above,rotate=-30] {$1: \frac{1}{2} - \epsilon$};
\draw[->] (q22t) -- (q23t) node[midway,above] {$0,1: \frac{1}{2} - \epsilon$};
\draw[->] (q23t) -- (q24t) node[midway,above] {$0,1: \frac{1}{2} - \epsilon$};

\draw[->] (q21f) -- (q22f) node[midway,above] {$1: \frac{1}{2} - \epsilon$};
\draw[->] (q21f) -- (q22t)  node[midway,above,rotate=30] {$0: \frac{1}{2} - \epsilon$};
\draw[->] (q22f) -- (q23f) node[midway,above] {$0: \frac{1}{2} - \epsilon$};
\draw[->] (q22f) -- (q23t) node[midway,above,rotate=30] {$1: \frac{1}{2} - \epsilon$};
\draw[->] (q23f) -- (q24f) node[midway,above] {$0: \frac{1}{2} - \epsilon$};
\draw[->] (q23f) -- (q24t) node[midway,above,rotate=30]{$1: \frac{1}{2} - \epsilon$};
\draw[->] (q24t) -- (q24f) node[midway,left] {$0,1: \epsilon$};
\path (q24f) edge [out=330,in=300,looseness=8] node[right] {$0,1: \frac{1}{2} - \epsilon$} (q24f);
\end{tikzpicture}
\end{center}
\caption{A graphical representation of the PFA constructed from $\F = (x_{1} \lor x_{2} \lor x_{3}) \wedge (\bar{x}_{2} \lor x_{3} \lor x_{4})$}
\end{figure*}
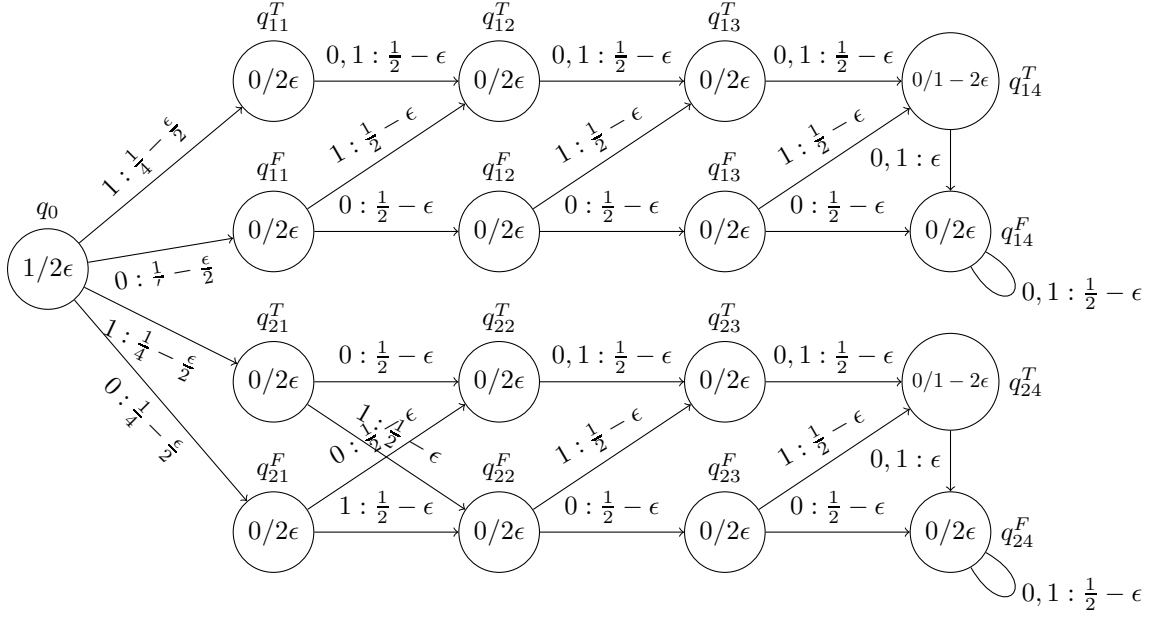
$\bullet$ \textbf{Construction of a RNN:} The RNN $R$ we'll construct is simple, and It generates the quantitative language $R(w) = 2(\frac{1}{2} - \epsilon)^{|w|} \epsilon$. More formally, our RNN is defined as: 
\begin{itemize}
 \item N = 2 (2 hidden neurons),
 \item $\begin{pmatrix} h_{n_{1}}^{(0)} \\ h_{n_{2}}^{(0)} \end{pmatrix} = \begin{pmatrix} 0 \\ 0 \end{pmatrix}$ 
 \item Transition matrices: $W_{in} = \begin{pmatrix} 1 & 0 \\ 0 & 1 \end{pmatrix}$; $W_{0} = W_{1} = W_{\$} =  \begin{pmatrix} 0 \\ 0 \end{pmatrix}$
 \item Output matrices: $O = \begin{pmatrix} 1 & 0 \\ 1 & 0 \\ 0 & 1 \end{pmatrix}$, $O' = \begin{pmatrix} \log_{2}{\frac{1-2 \epsilon}{4 \epsilon}} \\ \log_{2}{\frac{1-2 \epsilon}{4 \epsilon}} \\ 0 \end{pmatrix}$ where $\log_{2}(.)$ is the logarithm to the base 2 
 \end{itemize} 
 What's left is to show that $R(w)$ = $2(\frac{1}{2}$ - $\epsilon)^{|w|}$ defines a consistent language model:
 \begin{proposition}
  For any $\epsilon < \frac{1}{2}$, the weighted language model defined as $f(w) = 2 (\frac{1}{2} - \epsilon)^{|w|} \epsilon$ is consistent.
 \end{proposition}
 \begin{proof}
 We have:
 \begin{equation}
 \begin{aligned}
     \sum\limits_{w \in \Sigma^{*}} f(w) &= 2\epsilon \sum\limits_{n \in \mathbb{N}} \sum\limits_{w: |w| = n} (\frac{1}{2} - \epsilon)^{n} \notag \\
     &= 2 \epsilon \sum\limits_{n \in \mathbb{N}} (1-2 \epsilon)^{n} \notag \\
 \end{aligned}
  \end{equation}
 By applying the equality: $\sum\limits_{n \in \mathbb{N}} x^{n} = \frac{1}{1 - x}$ for any $|x| < 1$ on the sum present in the right-hand term of the equation above, we obtain the result.
 \end{proof}

\begin{proposition}
Let $\F$ be an arbitrary 3-SAT formula with $n$ variables and $k$ clauses. Let $\mathcal{A}$ be the PFA constructed from $\F$ by the procedure detailed above, the probabilistic language generated by $\mathcal{A}$ is given as: 
\begin{equation*}
    \mathcal{A}(w) =  \begin{cases} 
         2 (\frac{1}{2} - \epsilon)^{|w|} \epsilon & if ~ |w| < n\\
          2(\frac{1}{2} - \epsilon)^{|w|} \epsilon [\frac{N_{w}}{k}\frac{1-2\epsilon}{2\epsilon} + \frac{k-N_{w}}{k} ] & if ~ |w| = n \\ 
         2(\frac{1}{2} - \epsilon)^{|w|} \epsilon [\frac{N_{w_{:n}}}{k} \frac{2\epsilon}{1 - 2\epsilon}+ \frac{k-N_{w_{:n}}}{k}] & else\\
      \end{cases}
\end{equation*}
\end{proposition}

\begin{proposition}
 For any rational number $\epsilon < \frac{1}{4}$, there exists a rational number $c_{\epsilon}$ such that $\F$ is satisfiable if and only if $d_{\infty}(R, \mathcal{A}) > c$
\end{proposition}
\begin{proof}
 For any $w$ such that $|w| < n$, $|R(w) - \mathcal{A}(w)| = 0$ . \\
 For $|w| = n$, we have: 
 $$ |R(w) - \mathcal{A}(w)| = 2 \epsilon (\frac{1}{2} - \epsilon)^{n} \frac{N_{w}}{k} (\frac{1 - 4 \epsilon}{2 \epsilon}) $$
On the other hand, for $|w| > n$, we have: 
$$ |R(w) - \mathcal{A}(w)| =  2 \epsilon (\frac{1}{2} - \epsilon)^{|w|} \frac{N_{w}}{k} \frac{1 - 4 \epsilon}{1 - 2 \epsilon}$$
Note that we have for any $\epsilon < \frac{1}{4}$:
$$ \forall w \in \Sigma^{\geq n}:~~ |R(w) - \mathcal{A}(w)| \leq  |R(w_{:n}) - \mathcal{A}(w_{:n})| $$
 This means that, under this construction, the maximum is reached necessarily by a string whose length is exactly equal to $n$. Thus, we obtain:
$$d_{\infty}(R,\mathcal{A}) = 2\frac{\epsilon}{k} (\frac{1}{2} - \epsilon)^{n} \frac{1 - 4 \epsilon}{2 \epsilon} \max\limits_{w \in \Sigma^{n}} N_{w}$$

Note that $\F$ is satisfiable if and only if $\max\limits_{w \in \Sigma^{n}} N_{w} = k$. As a result, pick any $s \in [k-1,k)$, and define $c_{epsilon} =  2 \frac{\epsilon s}{k} (\frac{1}{2} - \epsilon)^{n} \frac{1 - 4 \epsilon}{2 \epsilon}$, the formula is satisfiable if and only if $d_{\infty}(R,\mathcal{A}) > c$. 
\end{proof}

\begin{theorem}
 The Tchebychev distance approximation problem between consistent RNN-LMs and PFAs in finite support is NP-Hard.
\end{theorem}
\begin{proof}
 We reduce the 3-SAT satisfiability problem to our problem. Let $\F$ be an arbitrary 3-SAT formula. Construct a PFA $\mathcal{A}$ and a RNN $R$ as specified previously. Choose a rational number $\epsilon < \frac{1}{4}$. Let $c_{\epsilon}>0$ be any rational number as specified in the proof of proposition 5.4, and $N = n+1$. By proposition 5.4, $\F$ is satisfiable if and only if $d_{\infty}(R,\mathcal{A}) > c$, which completes the proof
\end{proof}

\textbf{Remarks:} 
\begin{itemize}
    \item \textbf{NP-Hardness for LSTMs/GRUS:} Although our main focus in this article was on first-order weighted RNNs with one hidden layer, It is worth noting the the NP-Hardness reduction technique from 3-SAT problem we employed can easily be generalized to the case of LSTMs [], and GRUs, two widely used RNN architectures in practice. Indeed, our reduction relies on the construction of a \textit{memoryless} first-order RNN which makes abstraction of the state of the hidden units of the network, and exploits only the output bias vector $O'$. Hence, provided we have first-order output function for a LSTM (or GRU) architecture, the NP-Hardness result demonstrated above is easy to extend our proof to these architectures. 
    \item \textbf{Finite precision RNNs:} As said earlier in section II, a new line of work considered the analysis of the computational power of RNNs with bounded resources, which is a realistic condition in practice\cite{Merrill19}\cite{Weiss18a}. Broadly speaking, a finite-precision RNN is one whose weights and values of its hidden units are stored using a finite number of bits (See \cite{Korsky19} for further details). Under our construction, the same remark raised above about LSTMs/GRUs can be applied to RNNs with finite precision. In fact, It's easy to notice that, with a judicious choice of $\epsilon$, say $\frac{1}{10}$ (in which case $O'[0]=O'[1]=2$), the \textit{toy} memoryless RNN we constructed in the proof requires only 2 bits to encode a hidden unit and a weight value of the network. This shows that even approximating the Tcheybetchev distance in finite support between a language represented by PFA and that of a finite-precision first-order RNN with any computable activation function is also NP-Hard. 
\end{itemize}

\section{Conclusion and perspectives}
In this article, we investigated some computational problems related to the issue of approximating trained RNN language models by different classes of finite state automata. We proved that the equivalence problem of PDFAs/PFAs/WFAs and general weighted first-order RNN-LM with ReLu activation function with a single hidden layer is generally undecidable, and, as a result, trying to calculate any distance between them can't be computed. When restricting RNN-LMs to be consistent, we proved that approximating the Tcheybetchev distance between consistent RNN-LMs with general computable activation functions and PFAs is decidable, and that the same problem when restricted to a finite support is at least NP-Hard. Moreover, we gave arguments that the reduction strategy from 3-SAT problem we employed to prove this latter result makes this result generalizable to the class of LSTMs/GRUs and finite precision RNNs. \\ 
This work provides first theoretical results of examining equivalence and the quality of approximation problems between automata-based models and RNNs from a computational viewpoint. Yet, there are still many interesting problems on the issue that could motivate future research, such as: Is the equivalence problem between general RNN-LMs and different classes of finite state machines still undecidable when other highly non-linear activation functions (e.g. sigmoid, hyperbolic tangent ..) are used instead of ReLus? Is the equivalence problem between the cut-point language of an RNN-LM and a DFA decidable? If an RNN-LM is trained to recognize a language generated by a regular grammar, can we decide if its cut-point language is indeed regular? etc.   

\footnotesize
\bibliography{cap2020}

\end{document}